\documentclass[12pt]{article}

\usepackage[dvipsnames]{xcolor}
\usepackage{amsfonts}
\usepackage{amssymb}
\usepackage{amsmath}
\usepackage{scalefnt}
\usepackage[left=3.0cm,right=3.0cm,bottom=2.5cm,top=2.5cm]{geometry}
\usepackage{latexsym}
\usepackage{colortbl}
\usepackage{graphicx}
\usepackage{hyperref}
\usepackage{import}
\usepackage{lscape}
\usepackage{booktabs}

\usepackage[T1]{fontenc}    
\usepackage{hyperref}       
\usepackage{url}            
\usepackage{nicefrac}       
\usepackage{microtype}      
\usepackage{amsthm} 
\usepackage{paralist}
\usepackage{bm}
\usepackage[version=3]{mhchem}
\usepackage{pgfpages}
\usepackage{scalefnt}
\usepackage{subcaption}
\usepackage{lscape}
\usepackage{rotating}
\usepackage{mathtools}
\usepackage{multirow}
\usepackage{textcomp}
\usepackage{listings} 
\usepackage{authblk}
\usepackage[utf8]{inputenc}
\usepackage{caption}
\usepackage{amsmath,algorithm,tabularx}
\usepackage[noend]{algpseudocode}
\usepackage{diagbox}

\newtheorem{dfn}{Definition}
\newtheorem{thm}{Theorem}
\newtheorem{lem}{Lemma}
\newtheorem{rmk}{Remark}

\newtheorem{prop}{Proposition}

\newcommand{\Real}{\mathbb{R}}
\newcommand{\Set}[1]{\left\{ #1 \right\} }
\newcommand{\bE}{\mathbb{E}}
\newcommand{\Exp}[2]{\bE_{#1} \left[ #2 \right]}
\newcommand{\KLDiv}[2]{D_{\mathrm{KL}} \left[ #1 \| #2 \right] }

\newcommand{\Normal}[2]{\mathcal{N} \left( #1, #2 \right)}
\newcommand{\NormalPDF}[3]{\mathcal{N} \left( #1 ; #2, #3 \right)}
\newcommand{\vecZero}{\bm{0}}

\newcommand{\IdentityMatrix}{\bm{I}}
\newcommand{\Tr}[1]{\mathrm{Tr}#1}
\newcommand{\quadForm}[2]{#2^{\top} #1 #2}
\newcommand{\determinant}[1]{\left| #1 \right|}

\newcommand{\bx}{\bm{x}}
\newcommand{\bz}{\bm{z}}
\newcommand{\paramEnc}{\bm{\phi}}
\newcommand{\paramDec}{\bm{\theta}}

\newcommand{\originalZ}{\tilde{\bm{z}}}
\newcommand{\firstZ}{\bm{z}}
\newcommand{\secondZ}{\bm{z}^{\prime}}
\newcommand{\pairZ}{\bar{\bm{z}}}
\newcommand{\originalX}{\tilde{\bm{x}}}
\newcommand{\firstX}{\bm{x}}
\newcommand{\secondX}{\bm{x}^{\prime}}
\newcommand{\pairX}{\bar{\bm{x}}}
\newcommand{\dx}[1]{\mathrm{d} #1}
\newcommand{\vecX}{\bm{x}}
\newcommand{\muC}{\bm{\mu}_{c}}

\newcommand{\mupZ}{\bm{\mu}_{\pairZ}}
\newcommand{\sigmaC}{\bm{\Sigma}_{c}}

\newcommand{\sigmapZ}{\bm{\Sigma}_{\pairZ}}

\newcommand{\gmmCoef}[1]{\pi_{#1}}
\newcommand{\gmmMean}[1]{\bm{\mu}_{#1}}
\newcommand{\gmmVar}[1]{\bm{\Sigma}_{#1}}
\newcommand{\paramGMM}{\bm{\psi}}
\newcommand{\meanEnc}[1]{\bm{\mu}_{#1}}
\newcommand{\varEnc}[1]{\bm{\Sigma}_{ #1}}

\newcommand{\bmu}{\bm{\mu}}
\newcommand{\bSigma}{\bm{\Sigma}}
\newcommand{\muA}{\bm{\mu}_{a}}
\newcommand{\muB}{\bm{\mu}_{b}}

\newcommand{\sigmaA}{\bm{\Sigma}_{a}}
\newcommand{\sigmaB}{\bm{\Sigma}_{b}}
\newcommand{\augVar}{\bm{\Sigma}_{\mathrm{aug}}}

\title{
Robust VAEs via Generating Process of Noise Augmented Data
}
\author[1]{Hiroo Irobe\footnote{Equally contributed}}
\author[1]{Wataru Aoki$^\ast$}
\author[2]{Kimihiro Yamazaki}
\author[1]{Yuhui Zhang}
\author[1,3]{Takumi Nakagawa}
\author[1]{Hiroki Waida}
\author[2,3]{Yuichiro Wada}
\author[1,3]{Takafumi Kanamori\footnote{Corresponding author; E-mail: \texttt{kanamori@c.titech.ac.jp}}}
\affil[1]{\normalsize
            Department of Mathematical and Computing Science, 
            Tokyo Institute of Technology,
            2-12-1 Ookayama, Meguro-ku, 
            Tokyo,
            152-8550, 
            Japan
            }
\affil[2]{\normalsize
            Fujitsu Limited,
            4-1-1 Kamikodanaka, Nakahara-ku, Kawasaki-shi, 
            Kanagawa,
            211-8588, 
            Japan
            }
\affil[3]{\normalsize
            RIKEN Center for Advanced Intelligence Project,
            Nihonbashi 1-chome Mitsui Building, 15th floor, 1-4-1 Nihonbashi, Chuo-ku, 
            Tokyo,
            103-0027, 
            Japan
            }
\date{}

\begin{document}
\maketitle

\begin{abstract}
 Advancing defensive mechanisms against adversarial attacks in generative models is a critical research topic in machine learning. Our study focuses on a specific type of generative models - Variational Auto-Encoders (VAEs). Contrary to common beliefs and existing literature which suggest that noise injection towards training data can make models more robust, our preliminary experiments revealed that naive usage of noise augmentation technique did not substantially improve VAE robustness. In fact, it even degraded the quality of learned representations, making VAEs more susceptible to adversarial perturbations. This paper introduces a novel framework that enhances robustness by regularizing the latent space divergence between original and noise-augmented data. Through incorporating a paired probabilistic prior into the standard variational lower bound, our method significantly boosts defense against adversarial attacks. Our empirical evaluations demonstrate that this approach, termed Robust Augmented Variational Auto-ENcoder (RAVEN), yields superior performance in resisting adversarial inputs on widely-recognized benchmark datasets.
\end{abstract}

\section{Introduction}
\label{sec:introduction}
Generative models, such as Variational Auto-Encoders (VAEs)~\cite{DBLP:journals/corr/KingmaW13}, generative adversarial networks~\cite{NIPS2014_5ca3e9b1}, normalizing flows~\cite{pmlr-v37-rezende15}, diffusion models~\cite{DBLP:conf/iclr/0011SKKEP21}, and autoregressive models~\cite{pmlr-v48-oord16}, have garnered significant attention in machine learning research. These models are capable of generating novel content, including text~\cite{touvron2023llama}, images~\cite{DBLP:journals/corr/KingmaW13}, and videos~\cite{singer2023makeavideo}. In the case of VAEs, generation is achieved through an auto-encoder, with the training objective being the maximization of a variational lower bound~\cite{DBLP:journals/corr/KingmaW13}, hence its name. Recently, this methodology has been applied in drug discovery, particularly using a VAE-based model, cryoDRGN~\cite{zhong2021cryodrgn}, for reconstructing continuous distributions of 3D protein structures from 2D images collected by electron microscopes.

Existing research~\cite{Cemgil2020Adversarially} and~\cite{NEURIPS2022_39e9c591} show that the original VAE~\cite{DBLP:journals/corr/KingmaW13} struggles with reconstruction and downstream classification under adversarial inputs. Various methods~\cite{Cemgil2020Adversarially, pmlr-v108-eduardo20a, NEURIPS2022_39e9c591, willetts2021improving, NEURIPS2020_ac10ff19} have been proposed to address these issues. In~\cite{Cemgil2020Adversarially}, it was observed that the original VAE lacked robustness against inputs outside the empirical data distribution. Subsequently, they formulated a modified variational lower-bound that incorporates these inputs, aiming to construct a more robust VAE. In downstream classification tasks, their method showed remarkable efficacy. 
In study~\cite{NEURIPS2022_39e9c591}, a Markov chain Monte Carlo sampling approach was employed during the inference steps to counteract adversarial inputs, by shifting the latent representation of the attacks back to more probable regions within the posterior distribution. This technique~\cite{NEURIPS2022_39e9c591} effectively improved the model's robustness.

Besides the above successful approaches, another strategy, presented in~\cite{NEURIPS2019_335cd1b9}, involves using additive random noise $\bm{\epsilon}$ on original data $\originalX$ to fortify classifiers against adversarial inputs in supervised learning, as supported by theoretical and experimental evidence. Influenced by these findings, we conducted preliminary experiments on the MNIST dataset~\cite{lecun1998gradient}, comparing the performance of the following two VAE configurations in downstream classification tasks~\cite{Cemgil2020Adversarially}:
\begin{description}
    \item[(A)] a VAE trained exclusively on original dataset $\{\originalX_i\}$,
    \item[(B)] a VAE trained on both original $\{\originalX_i\}$ and corresponding noisy dataset $\{\firstX_i\}$, where $\firstX_i = \originalX_i + \bm{\epsilon_i}$, and $\bm{\epsilon_i}$ is random Gaussian noise: $\bm{\epsilon_i} \sim \Normal{\vecZero}{0.05^2 \IdentityMatrix}$.
\end{description}
We assessed performance using the encoder's linear classification accuracy against adversarial inputs~\cite{Cemgil2020Adversarially}. However, the results showed no significant improvement in classification accuracies, suggesting that the
naive usage of the noisy data $\firstX$ in (B) was insufficient to enhance the original VAE's robustness. 
\begin{figure}[htbp]
\centering
\includegraphics[scale=0.5]{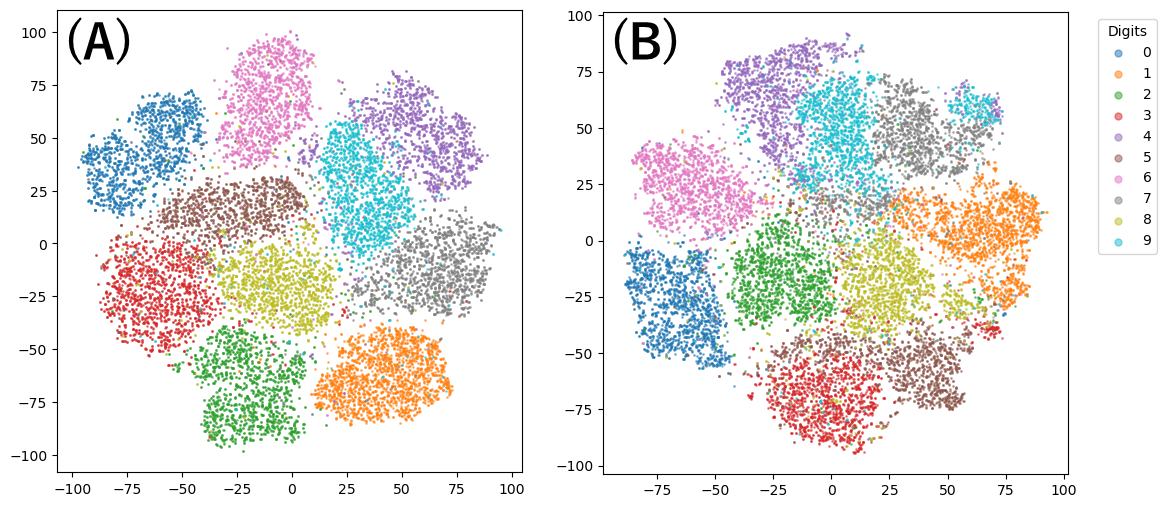}
\caption{
t-SNE visualization of latent variables by trained VAE on test MNIST: (A) trained on original MNIST training data $\{\originalX_i\}$, (B) trained on both original and noisy data $\{\originalX_i\} \cup \{\firstX_i\}$, where $\firstX_i = \originalX_i + \bm{\epsilon}_i$. Colors denote class labels on the top-right. 
}
\label{fig: visualization of latent via vae with noisy data}
\end{figure}
To understand the cause of the failure, we analyzed the latent variables $\originalZ$ and $\firstZ$, outputs of the encoder on inputs $\originalX$ and $\firstX$ respectively; the t-SNE visualization~\cite{vanDerMaaten2008} of Figure~\ref{fig: visualization of latent via vae with noisy data} also shows that the representations of (B) are similar to (A). 
We then discovered unexpected behavior in the VAE under condition (B): the average squared Euclidean distance between $\originalZ$ and $\firstZ$ was larger compared to (A), despite the inclusion of noisy data in (B)'s training dataset; see the first and second columns in Table~\ref{tab: verification of our hypothesis}.
See further details of our preliminary experiments in Appendix~\ref{apped: details preliminary experiments}.

We \emph{hypothesize} that this failure partly stems from the distant representation of the pair $(\originalZ,\firstZ)$, which hampers the enhancement of robustness. In this study, we aim to construct a robust VAE by 
bringing 
the variables $\originalZ$ and $\firstZ$ closer together in the latent space. To achieve this, we integrated a probabilistic model of the paired variables~\cite{haochen2021provable} into the prior of the original variational lower-bound~\cite{DBLP:journals/corr/KingmaW13}. This  model illustrates the generation mechanism of the paired variables from original data. 

Our three key contributions are summarized below:
\begin{description}
    \item[1)] We derived a novel variational lower-bound by incorporating the generating process for a robust VAE. The proposed bound is computationally efficient due to its closed-form structure, and its simplicity allows seamless integration with existing VAE extensions like $\beta$-VAE~\cite{higgins2017betavae} and VaDE~\cite{ijcai2017p273}, as well as methods like~\cite{NEURIPS2022_39e9c591} and~\cite{NEURIPS2020_ac10ff19} to further enhance the robustness. We name our method the Robust Augmented Variational auto-ENcoder (\emph{RAVEN}).

    \item[2)] Our analysis of this bound reveals that it is a generalization of the original bound~\cite{DBLP:journals/corr/KingmaW13}, and provides an intuitive understanding of how it effectively narrows the distance between the paired representational variables $(\originalZ,\firstZ)$.

    \item[3)] We empirically prove that our hypothesis is correct by results of our numerical experiments: RAVEN shows superior resilience against various adversarial attacks across common datasets without compromising on reconstruction quality. The superior performance represents a significant advancement in the field, enhancing both the robustness and versatility of VAEs.
\end{description}

\begin{table}[!t]
\captionsetup[table]{textfont=normalfont}
\caption{
Mean $\pm$ std on $\{\|\originalZ_i - \firstZ_i\|_2^2\}_{i=1}^{10^4}$, where $\originalZ$ (resp. $\firstZ$) is representation of original MNIST test data $\originalX$ (resp. the corresponding noisy data $\firstX = \originalX + \bm{\epsilon}$), obtained by an encoder. "VAE" and "VAE (noise)" correspond to (A) and (B).
}
\centering
\scalebox{1}{
\begin{tabular}{ccc}
    \toprule
     VAE & VAE (noise) & Ours  \\
    \midrule
    $0.89 \pm 1.29$ & $3.24 \pm 5.12$ & $0.53 \pm 0.36$ \\
    \bottomrule
\end{tabular}
}
\label{tab: verification of our hypothesis}
\end{table}

\section{Related Work}
\label{sec:related work}
In Section~\ref{subsec:vae}, we review details of the original VAE~\cite{DBLP:journals/corr/KingmaW13}, which serves as a foundational method for our study.
Subsequently, in Section~\ref{subsec:other methods}, we briefly review existing methods to build a robust VAE~\cite{willetts2021improving} and~\cite{NEURIPS2020_ac10ff19}, as introduced at the second paragraph of Section~\ref{sec:introduction}. At last, in Section~\ref{subsec:se}, we describe details of the Smooth Encoder (SE)~\cite{Cemgil2020Adversarially}, given its relative similarity to our proposed method. The distinctions between this method and our own to ours are clarified at the conclusion of this section.

In the following, for a random variable $\bz$,
let $\NormalPDF{\bz}{\bmu}{\bSigma}$ denote a multivariate Gaussian distribution with mean $\bmu$ and variance matrix $\bSigma$: 
\begin{equation}
    \label{eq: multivariate normal distribution}
    \NormalPDF{\bz}{\bmu}{\bSigma} = \frac{\exp \left( -\frac{1}{2} (\bz - \bmu)^\top \bSigma^{-1} (\bz - \bmu) \right)}{\sqrt{\determinant{2 \pi \bSigma}}}, 
\end{equation}
and we use notation $\bz \sim \Normal{\bmu}{\bSigma}$.

\subsection{Variational Auto-Encoder}
\label{subsec:vae}
The VAE stands out as a widely recognized unsupervised generative model, based on an auto-encoder architecture. The auto-encoder is bifurcated into two components: an encoder $\mathrm{Enc}_{\paramEnc}$ and the decoder $\mathrm{Dec}_{\paramDec}$, where $\paramEnc$ and $\paramDec$ are a set of the trainable parameters. 
Given a data $\bx\in\Real^{D}$, the encoder returns $(\meanEnc{\bx}, \bm{\sigma}_{\bx})=\mathrm{Enc}_{\paramEnc}(\bx), \meanEnc{\bx}\in\mathbb{R}^d, \bm{\sigma}_{\bx}\in\mathbb{R}^d$ to define the latent variable $\bz \in\mathbb{R}^d$ by $\bz = \meanEnc{\bx} + \bm{\sigma}_{\bx} \odot \bm{\varepsilon}$, where $\odot$ is the Hadamard product and $\bm{\varepsilon}$ is a $d$-dimensional random vector, whose distribution is a unit Gaussian $\bm{\varepsilon} \sim \Normal{\vecZero}{\IdentityMatrix}$. The two symbols $\vecZero\;\text{and}\;\IdentityMatrix$ respectively express the zero mean vector and the variance-covariance matrix defined by an identity matrix, respectively. By the definition, $\bz\sim\Normal{\bmu_{\bx}}{\bSigma_{\bm x}}$, $\bSigma_{\bm x}=\mathrm{diag}\left(\bm{\sigma}_{\bx} \odot \bm{\sigma}_{\bx}\right)$.

The auto-encoder is trained under an objective that maximizes the log-likelihood $\log p_{\bm \theta}(\bx)$. Since the direct maximization  is intractable, the authors of~\cite{DBLP:journals/corr/KingmaW13} proposed the variational lower-bound to indirectly maximize $\log p_{\bm \theta}(\bx)$, where $p_{\paramDec}(\bx) = \int p_{\paramDec}(\bx | \bz) p(\bz) \dx{\bz}$,  $p_{\paramDec}(\bx | \bz)$ is a likelihood defined by the decoder $\mathrm{Dec}_{\paramDec}$, and $p(\bz)$ is the prior distribution. The lower-bound $\mathcal{L}\left(\firstX; \bm{\phi}, \bm{\theta}\right)$ has a form of
\begin{equation}
\label{eq: original vae lower-bound}
    \mathcal{L}\left(\firstX; \bm{\phi}, \bm{\theta}\right) := \underbrace{\Exp{q_{\paramEnc}(\bz | \bx)}{\log p_{\paramDec}(\bx | \bz)}}_\text{Reconstruction error} - \underbrace{\KLDiv{q_{\paramEnc}(\bz | \bx)}{p(\bz)}}_\text{Regularization term},
\end{equation}
where $q_{\paramEnc}(\bz | \bx)$ is the approximated posterior distribution defined by $\mathrm{Enc}_{\paramEnc}$. The symbol $D_{\textrm{KL}}$ means the Kullback-Leibler (KL) divergence. In the original VAE, the normal Gaussian is set to the prior, and  the KL term can be expressed as a closed form~\cite{DBLP:journals/corr/KingmaW13}. Usually, the reconstruction term is approximately computed by a sample from the approximated posterior.

\subsection{Existing Methods for Building Robust VAE}
\label{subsec:other methods}
The researchers in~\cite{willetts2021improving} introduced the regularization method called \emph{total correlation} to the original VAE, which serves to smoothen the embeddings in the latent space. Additionally, by introducing a hierarchical VAE, it was able to achieve both adversarial robustness and high reconstruction quality. 
The authors of~\cite{NEURIPS2020_ac10ff19} suggest that VAEs do not consistently encode typical samples generated from their own decoder. They propose an alternative construction of the VAE with a Markov chain alternating between encoder and decoder, to enforce their consistency. Their method is reported to perform well in downstream classification tasks with adversarial inputs.

\subsection{Smooth Encoder}
\label{subsec:se}
As briefly mentioned in Section~\ref{sec:introduction}, the authors of~\cite{Cemgil2020Adversarially} experimentally found that the original VAE was not robust against an input out of the the empirical distributional support. To solve the problem, 
they proposed SE. This method assumes that the types of future attacks are known beforehand, thereby generating adversarial data $\firstX$ from the original data $\originalX$ tailored to these expected attacks.  The method then employs a regularization to ensure the proximity between the representations of the original and the adversarial data. 
This subsequently leads to a notable minimization of the entropy-regularized Wasserstein distance across the data representations. In SE's framework, the derivation of a lower bound on the joint distribution of $(\originalX, \firstX)$ plays a crucial role, leading to a corresponding lower bound on the log-likelihood by integrating out $x'$.

In contrast to SE, our method is not reliant on foreknowledge of attack types, thus simplifying its computational process and eliminating the iterative gradient calculations required for creating adversarial data in SE. Additionally, our method is based on  a variational lower bound directly on the joint distribution over the augmented data pair, as we will soon discuss.

\section{Proposed Method}
\label{sec:proposed method}
Let $\firstX$ and $ \secondX$ be two augmented data instances generated from the original data $\originalX$. The joint distribution of this pair, denoted as $p\left(\firstX, \secondX\right)$, is assumed to have the following distribution~\cite{haochen2021provable}:
\begin{equation}
    \label{eq:haochen data augmentation in original space}
    \begin{split}
        p(\firstX, \secondX) 
        &= \Exp{\originalX \sim p(\originalX)}{A(\firstX | \originalX) A(\secondX | \originalX)} \\
        &= \int A(\firstX | \originalX) A(\secondX | \originalX) p(\originalX) \dx{\originalX}, 
    \end{split}
\end{equation}
where 
$p(\originalX)$ is the original data distribution, and 
$A(\cdot | \originalX)$ is the augmented data distribution conditioning on $\originalX$. 

Similarly to the original VAE objective in Eq.\eqref{eq: original vae lower-bound}, we consider a variational lower-bound of log-likelihood of the joint distribution over $\pairX = \left(\firstX, \secondX\right) $, as follows:
\begin{equation}
\label{eq:elbo paired joint}
    \begin{split}
        \log p_{\paramDec}(\pairX) &\geq \mathcal{L}\left(\Bar{\firstX}; \bm{\phi}, \bm{\theta}\right) \\
        &= \underbrace{\Exp{\pairZ \sim q_{\paramEnc}(\pairZ | \pairX)}{\log p_{\paramDec}(\pairX | \pairZ)}}_{\text{\rm \textcircled{\scriptsize 1}}} \underbrace{- \KLDiv{q_{\paramEnc}(\pairZ | \pairX)}{p(\pairZ)}}_{\text{\rm \textcircled{\scriptsize 2}}},
    \end{split}
\end{equation}
where $\pairZ = \left( \firstZ, \secondZ \right)$ are the corresponding latent variables to $\pairX = (\firstX, \secondX)$.

Inspired by the formulation presented in Eq.\eqref{eq:haochen data augmentation in original space}, we reconsider the generative process of the augmented data from the latent space perspective. We assume that, the corresponding latent variables, $\firstZ$ and $\secondZ$, follow some conditional distribution given the original one  $a(\cdot | \originalZ)$. By adopting Gaussian distributions for $a(\cdot | \originalZ)$, we achieve a closed-form derivation for Eq.\eqref{eq:elbo paired joint}. Our construction of a robust VAE will be based on the maximization of the novel lower bound. 
\begin{dfn}
\label{dfn:latent generating process}
(Variant of Example~3.8 in~\cite{haochen2021provable})
 The prior $p(\pairZ)$ in Eq.\eqref{eq:elbo paired joint} is given by
\begin{equation}
    \label{eq:haochen data augmentation in latent space}
    p(\pairZ) 
    = \int a(\firstZ | \originalZ) a(\secondZ | \originalZ) p(\originalZ)  \dx{\originalZ},
\end{equation} 
where $p(\originalZ)=\NormalPDF{\originalZ}{\vecZero}{\IdentityMatrix}$, $a(\firstZ | \originalZ) = \NormalPDF{\firstZ}{\originalZ}{\augVar}$, and $a(\secondZ | \originalZ) = \NormalPDF{\secondZ}{\originalZ}{\augVar}$.
\end{dfn}
Note that we can recover Definition~\ref{dfn:latent generating process} by setting a normal Gaussian to the original data distribution of Example~3.8 in~\cite{haochen2021provable}.

\subsection{Proposed Variational Lower-Bound}
\label{subsec:proposed lower-bound}
\begin{prop}
\label{prop: prior latent joint as closed form}
The prior $p(\pairZ)$ in Definition~\ref{dfn:latent generating process} has a form of 
\begin{equation} 
\label{equation:augmented prior of normal gaussian}
    \NormalPDF{\vecZero}{\firstZ - \secondZ}{2 \augVar}  \, \NormalPDF{\frac{\firstZ + \secondZ}{2}}{\vecZero}{\IdentityMatrix + \frac{1}{2} \augVar},
\end{equation}
where the first and second terms are defined by Eq.\eqref{eq: multivariate normal distribution}.
\end{prop}
The proof is given in Appendix~\ref{append:proof of closed form based on N(0,I) p_tilde_z}.

\begin{thm}
\label{thm: proposed variational bound}
Assume the independence for $q_{\paramEnc}(\pairZ | \pairX)$ of Eq.\eqref{eq:elbo paired joint} as follows: $q_{\paramEnc}(\pairZ | \pairX) = q_{\paramEnc}(\firstZ | \firstX) \cdot q_{\paramEnc}(\secondZ | \secondX)$, i.e., $q_{\paramEnc}(\pairZ | \pairX) = \NormalPDF{\firstZ}{\meanEnc{\firstX}}{\varEnc{\firstX}} \cdot \NormalPDF{\secondZ}{\meanEnc{\secondX}}{\varEnc{\secondX}}$; see the definitions of $\meanEnc{\firstX}$ and $\varEnc{\firstX}$ in Section~\ref{subsec:vae}.
Then, \text{\rm \textcircled{\scriptsize 2}} of Eq.\eqref{eq:elbo paired joint} can be written as a closed form:
\begin{equation}
\label{eq: closed form Dkl with standard normal}
    \begin{split}
         &\text{\rm \textcircled{\scriptsize 2}}\!=\!- \frac{1}{4} \Bigl\{ \Tr{\left(\left(
         \augVar^{-1} + \left(2\IdentityMatrix + \augVar  \right)^{-1} \right) (\varEnc{\firstX} + \varEnc{\secondX})\right)} + \text{\rm \textcircled{\scriptsize 3}} \Bigr\}\\ 
         &\;\;\;\;\;\;\;\;\;\;\; +\frac{1}{2} \Bigl\{ \log \determinant{\varEnc{\firstX} } + \log \determinant{\varEnc{\secondX}}  + 2d\left(1 - \log 2 \right) -  \log \determinant{\augVar} - \log \determinant{2\IdentityMatrix + \augVar} \Bigr\},
    \end{split}
\end{equation}
where $\text{\rm \textcircled{\scriptsize 3}}=\quadForm{\augVar^{-1}}{(\meanEnc{\firstX} - \meanEnc{\secondX})} + \quadForm{\left(2 \IdentityMatrix + \augVar \right)^{-1}}{(\meanEnc{\firstX} + \meanEnc{\secondX})}$.
Moreover, assume that $p_{\paramDec}(\pairX | \pairZ) = p_{\paramDec}(\bm{x}|\bm{z})\cdot p_{\paramDec}(\bm{x}^\prime|\bm{z}^\prime)$. 
Then, based on \text{\rm \textcircled{\scriptsize 2}} of Eq.\eqref{eq: closed form Dkl with standard normal}, 
$\mathcal{L}\left(\Bar{\firstX}; \bm{\phi}, \bm{\theta}\right)$ of Eq.\eqref{eq:elbo paired joint} is equal to
\begin{equation}
    \label{eq:likelihood decomposition via independence property}
    \mathbb{E}_{\firstZ \sim q_{\paramEnc}(\firstZ | \firstX)}\left[ \log p_{\paramDec}(\bm{x}|\bm{z}) \right] + \mathbb{E}_{\secondZ \sim q_{\paramEnc}(\secondZ | \secondX)}\left[ \log p_{\paramDec}(\bm{x}^\prime|\bm{z}^\prime)\right] + \text{\rm \textcircled{\scriptsize 2}}.
\end{equation}
\end{thm}
The proof sketch is given in Section~\ref{subsec:proofs}. 
We also derived the variational lower bound of Eq.\eqref{eq:elbo paired joint} in the case that $p(\originalZ)$ of Definition~\ref{dfn:latent generating process} is the Gaussian mixture model; see Appendix~\ref{append: GMM case}. 
\begin{rmk}
    We remark the connection between our proposed bound in Theorem~\ref{thm: proposed variational bound} and the standard bound Eq.\eqref{eq: original vae lower-bound}. 
    We select a "typical" input data for $\secondX$ so that $\meanEnc{\secondX} = \vecZero, \varEnc{\secondX} = \IdentityMatrix$, and set $\augVar = \frac{\sqrt{5}-1}{2} \IdentityMatrix$. Disregarding constants that do not depend on network parameters or the input $\firstX$, the KL term of Eq.\eqref{eq:elbo paired joint} simplifies to $-\frac{1}{2} \left(\meanEnc{\firstX}^\top \meanEnc{\firstX} + \Tr{ (\varEnc{\firstX})} - \log \determinant{\varEnc{\firstX}} \right)$, aligning with the KL term from the standard VAE in Eq.\eqref{eq: original vae lower-bound}.
\end{rmk}

\begin{rmk}
    Regarding \text{\rm \textcircled{\scriptsize 3}} in Eq.\eqref{eq: closed form Dkl with standard normal} - the terms related to the mean vectors $\boldsymbol{\mu}_{\boldsymbol{x}}$ and $\boldsymbol{\mu}_{\boldsymbol{x}^\prime}$ outputted by the encoder, when $\augVar \ll \IdentityMatrix $: $\augVar$ is substantially less than $\IdentityMatrix$ in terms of e.g., maximum eigenvalues, it follows that $\augVar^{-1} \gg (2\IdentityMatrix + \augVar)^{-1}$. Consequently, the first term predominantly enforces $\meanEnc{\firstX}$ and $\meanEnc{\secondX}$ to converge. Furthermore, as shown in the following proposition, \text{\rm \textcircled{\scriptsize 3}} enables the model to learn the appropriate distance for representations from the origin.
\end{rmk}

\begin{prop}
\label{prop: mean vector terms}
The term \text{\rm \textcircled{\scriptsize 3}} in Eq.\eqref{eq: closed form Dkl with standard normal}  can  further be derived as 
\begin{equation}
\label{eq: waida-kun-derivation}
        \text{\rm \textcircled{\scriptsize 3}} = \boldsymbol{\mu}_{\boldsymbol{x}}^{\top} \augVar^{-1} \boldsymbol{\mu}_{\boldsymbol{x}}+\boldsymbol{\mu}_{\boldsymbol{x}^{\prime}}^{\top} \augVar^{-1} \boldsymbol{\mu}_{\boldsymbol{x}^{\prime}} \\
        -\left(\boldsymbol{\mu}_{\boldsymbol{x}}+\boldsymbol{\mu}_{\boldsymbol{x}^{\prime}}\right)^{\top} \augVar^{-1}\left(2^{-1} \boldsymbol{I}+\augVar^{-1}\right)^{-1} \augVar^{-1}\left(\boldsymbol{\mu}_{\boldsymbol{x}}+\boldsymbol{\mu}_{\boldsymbol{x}^{\prime}}\right).
\end{equation}
\end{prop}
See the proof in Appendix \ref{append:proof of mean vectors terms}. In the right hand of Eq.\eqref{eq: waida-kun-derivation}, the first and second terms force both $\boldsymbol{\mu}_{\boldsymbol{x}}$ and $\boldsymbol{\mu}_{\boldsymbol{x}^{\prime}}$ to be the zero vector, while the last term keeps these vectors away and avoid to contract to the zero vector. Thus, the model has to learn to balance the two.

\subsection{Proof Sketch for Theorem~\ref{thm: proposed variational bound}}
\label{subsec:proofs}
\begin{lem}
\label{lemma: quadOfMeanCov} (Corollary 3.2b.1 of~\cite{mathai1992quadratic})
Consider the random variable $\bm{x} \sim \mathcal{P}$, whose mean and the variance matrix are $\bm{\mu}$ and $\bm{\Sigma}$, respectively. For a symmetric matrix $\bm{A}$, the following equation holds:
$\Exp{\bm{x} \sim \mathcal{P}}{\quadForm{\bm{A}}{\bm{x}}} = \Tr{\left(\bm{A} \bm{\Sigma} \right)} + \quadForm{\bm{A}}{\bm{\mu}}$. Then, with a constant vector $\bm{a}$, the following holds: $\Exp{\bm{x} \sim \Normal{\bm{\mu}}{\bm{\Sigma}}}{\quadForm{\bm{A}}{\left( \bm{x} - \bm{a}\right)}} =\Tr{\left(\bm{A} \bm{\Sigma} \right)} + \quadForm{\bm{A}}{\left(\bm{\mu} - \bm{a}\right)}$.
\end{lem}

\begin{figure*}[!t]
\centering
\includegraphics[scale=0.85]{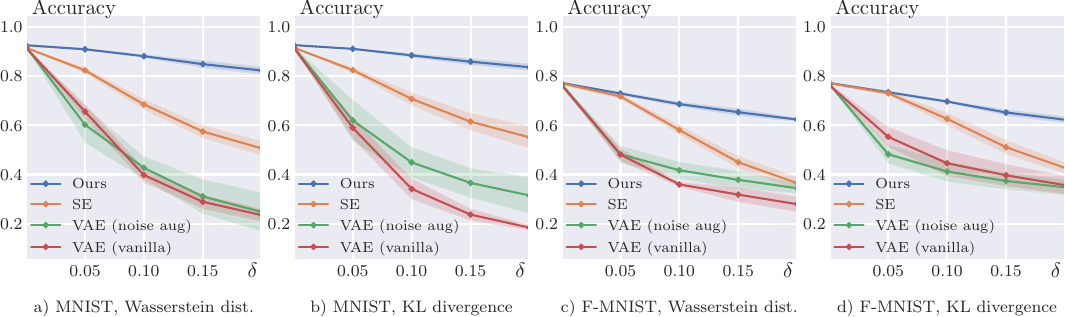}
\caption{Attack budget $\delta$ (horizontal axis) versus classification accuracy (vertical axis) on MNIST (a, b) and Fashion-MNIST (c, d) datasets. Our proposed method (blue) significantly outperforms baseline methods under severe adversarial attacks, and is on par with them as the adversarial signal vanishes. Shaded area indicates standard derivations over 5 runs.}
\label{fig: expr}
\end{figure*}

\begin{lem}
\label{proposition: CrossEntropy}
For the two $d$-dimensional multivariate normal distributions $\Normal{\bmu_{a}}{\bSigma_{a}}$ and $\Normal{\bmu_{b}}{\bSigma_{b}}$,
consider the cross-entropy $\Exp{\bx \sim \Normal{\bmu_{a}}{\bSigma_{a}}}{-\log \NormalPDF{\bx}{\bmu_{b}}{\bSigma_{b}} }$. Then, it is written as
$\frac{1}{2} \{ \Tr ( \bSigma_{a} \bSigma_{b}^{-1} ) + \quadForm{\bSigma_{b}^{-1}}{(\bmu_{a} - \bmu_{b})} + \log \determinant{\bSigma_{b}} + d \log 2\pi \}$.
Additionally, for the $d$-dimensional normal distribution $\Normal{\bmu}{\bSigma}$, the Shannon entropy of $\Normal{\bmu}{\bSigma}$, i.e., $H := \Exp{\bx \sim \Normal{\bmu}{\bSigma}}{-\log \NormalPDF{\bx}{\bmu}{\bSigma} }$ is written as $\frac{1}{2} \Set{d + d \log 2\pi + \log \determinant{\bSigma} }$.
\end{lem}
Lemma~\ref{proposition: CrossEntropy} can be immediately proved by using Lemma~\ref{lemma: quadOfMeanCov}, as shown in Appendix~\ref{append:proof lemma2}.

From Lemmas~\ref{lemma: quadOfMeanCov} and~\ref{proposition: CrossEntropy}, we have
\begin{equation}
\label{eq: eq from lemma 1 and 2}
    \begin{split}
        &\Exp{\firstZ \sim \Normal{\meanEnc{\firstX}}{\varEnc{\firstX}}}{\log \NormalPDF{\vecZero}{\firstZ - \secondZ}{2\augVar}}\\
        &=-\frac{1}{4} \left(\Tr{\left( \augVar^{-1} \varEnc{\firstX}\right)} + \quadForm{\augVar^{-1}}{(\meanEnc{\firstX} - \secondZ)} \right) - \frac{1}{2} \log \determinant{4\pi \augVar}.
    \end{split}
\end{equation}
Since $q_{\paramEnc}(\pairZ | \pairX) = \NormalPDF{\firstZ}{\meanEnc{\firstX}}{\varEnc{\firstX}} \cdot \NormalPDF{\secondZ}{\meanEnc{\secondX}}{\varEnc{\secondX}}$, using Eq.\eqref{eq: eq from lemma 1 and 2}, we have
\begin{equation}
\label{eq: component eq1 of Dkl}
    \begin{split}
        &\Exp{\pairZ \sim q_{\paramEnc}(\pairZ | \pairX)}{\log \NormalPDF{\vecZero}{\firstZ - \secondZ}{2\augVar}} \\
        &=-\frac{1}{4} \Bigl\{ \Tr{(\augVar^{-1}\varEnc{\firstX})} + \Exp{\secondZ \sim \Normal{\meanEnc{\secondX}}{\varEnc{\secondX}}}{ \quadForm{\augVar^{-1}}{(\meanEnc{\firstX} - \secondZ)}} \Bigr\} -\frac{1}{2} \log \determinant{4 \pi \augVar}\\
        &= -\frac{1}{4} \Tr{\left(\augVar^{-1}(\varEnc{\firstX}\!+\!\varEnc{\secondX}) + \quadForm{\augVar}{(\meanEnc{\firstX}\!-\!\meanEnc{\secondX})} \right)} - \frac{1}{2}\left(\log \determinant{\augVar} + d \log 4 \pi \right).
    \end{split}
\end{equation}
Like deriving Eq.\eqref{eq: component eq1 of Dkl}, we have
\begin{equation}
\label{eq: component eq2 of Dkl}
    \begin{split}
        &\Exp{\pairZ \sim q_{\paramEnc}(\pairZ | \pairX) }{\log \Normal{\frac{\bm{z}+\bm{z}'}{2}}{ \bm{0}, \IdentityMatrix+\frac{1}{2} \augVar }}  \\
        & = -\frac{1}{4}\Bigl\{\Tr{ \left(\left( \varEnc{\firstX} + \varEnc{\secondX} \right) ( 2\IdentityMatrix + \augVar )^{-1}\right) } + \left( \meanEnc{\firstX}+\meanEnc{\secondX} \right)^\top ( 2\IdentityMatrix + \augVar )^{-1} ( \meanEnc{\firstX} + \meanEnc{\secondX} )  \Bigr\} \\
        &\quad\quad-\frac{1}{2}\left(\log \determinant{\augVar + 2\IdentityMatrix} + d \log 4 \pi \right).
    \end{split}
\end{equation}
In Eq.\eqref{eq:elbo paired joint}, by the definition of the KL divergence, the part \text{\rm \textcircled{\scriptsize 2}} equals to 
$$\underbrace{\Exp{q_{\paramEnc}(\pairZ | \pairX)}{\log p_{\bm{\theta}}(\pairZ)}}_{\text{\rm \textcircled{\scriptsize 4}}} + \underbrace{H\left(q_{\paramEnc}(\pairZ | \pairX) \right)}_{\text{\rm \textcircled{\scriptsize 5}}},$$
where $H$ means the Shannon entropy. Moreover, from Proposition~\ref{prop: prior latent joint as closed form}, the part \text{\rm \textcircled{\scriptsize 4}} is decomposed into the summation of Eq.\eqref{eq: component eq1 of Dkl} and Eq.\eqref{eq: component eq2 of Dkl}. Finally, by applying Lemma~\ref{proposition: CrossEntropy} to \text{\rm \textcircled{\scriptsize 5}}, we obtain Eq.\eqref{eq: closed form Dkl with standard normal}.

\section{Numerical Experiments}
\label{sec:numerical experiments}
Toward a practical evaluation of our method, we conducted experiments on MNIST~\cite{lecun1998gradient} and Fashion MNIST~\cite{xiao2017fashion} (F-MNIST) datasets while keeping our focus on model robustness. 

The efficiency of our proposed model against adversarial attacks is highlighted in those experiments.
We report the corresponding details and results in the following.

\subsection{Experiment Setup}
\label{sec: ex-a}
\subsubsection{ Adversarial attacks } 
\label{subsubsec: Adversarial attacks}
An adversarial attack is some carefully crafted alternations of the input data,
designed to mislead models into yielding erroneous predictions or classifications~\cite{43405}.
In our study, we focus on small perturbations on the input data that can provoke substantial deviations in representations.
Given attackers an accessible encoder model $q_{\paramEnc}(\cdot | \firstX) $, and a specified attack budget $\|\bm{\epsilon}\|_p \le \delta$,
their goal is to induce the maximum change in model representations,
measured by some dissimilarity function $\Delta$: 
\begin{equation}    
\bm{\epsilon} = \arg \max_{\|\bm{\epsilon}\|_p \le \delta} \Delta \left[ q_{\paramEnc}(\cdot | \firstX), q_{\paramEnc}(\cdot | \firstX + \bm{\epsilon})  \right] .
\label{eq: adv-attack}
\end{equation}
In Eq.\eqref{eq: adv-attack}, we adopted $p = \infty$ hence the maximum norm $\|\cdot\|_{\infty}$. $\Delta$ takes Wasserstein distance or KL divergence.

Therefore, throughout our experiments, there are 2 types of adversarial attacks being considered in total.

To optimize Eq.\eqref{eq: adv-attack} w.r.t. $\bm{\epsilon}$, we apply the Projected Gradient Descent (PGD) algorithm with 50 iterations and stepsize $\delta/25$, maximizing either Wasserstein distance or KL divergence between the encoded representation posteriors.

\subsubsection{Experiment protocols}
\label{subsubsec: Experiment protocols}
Referring the experiment protocols in \cite{Cemgil2020Adversarially}, our experiments with adversarial attack proceeds as follows:

As the first step, we train the VAE independently of downstream tasks until model converges.
The encoder and the decoder are symmetric 3-layer MLPs with PReLU activations~\cite{he2015delving}, without normalization layers.
The model is trained with RAdam optimizer~\cite{Liu2020On} and learning rate 0.001.

We then freeze the encoder parameters, and extract representation $\firstZ_i$ for each training data $\firstX_i$. 
The mean given by the encoder will be the representation $\firstZ_i$:
\begin{equation}
   \bm{z}_i = \arg\max_{\firstZ \in \Real^d} \NormalPDF{\bz}{\bmu_{\firstX_i}}{\bSigma_{\firstX_i}}
   = \meanEnc{\firstX_i}.
\end{equation}

Secondly, following the common linear evaluation method, we train a simple linear classifier on the pairs of fixed representations $\{\bm{z}_i\}_{i=1}^n$ and corresponding labels $\{y_i\}_{i=1}^n$ via the cross-entropy loss.

Finally, we assess model robustness by generating adversarial examples for each test instance using Eq.\eqref{eq: adv-attack}.
We report their test-set classification accuracy under such attacks in our results.

\begin{table}[!t]
\caption{
Reconstruction performance results of proposed method and various baselines. 
Numbers indicate mean and standard derivations over 5 runs.
The best performance is in bold font. Results with $\delta=0$ in Figure~\ref{fig: expr} are reported in this table.
}
\centering
\scalebox{1}{
\begin{tabular}{ccccc}
    \toprule
    & model & Accuracy (\%) $\uparrow$ & MSE $\downarrow$ & FID $\downarrow$ \\
    \midrule
    \multirow{4}{*}{\rotatebox{90}{MNIST}}
        & VAE & 91.45  $\pm$ 0.26 & 862.114 $\pm$ 10.96& 21.766 $\pm$ 0.27 \\
        & VAE (noise) & 91.84 $\pm$ 0.45 & 916.383 $\pm$ \phantom{1}5.43& 23.990 $\pm$ 0.86  \\
        & SE & 91.44 $\pm$ 0.19 & 876.454 $\pm$ 10.82& 22.434 $\pm$ 0.70  \\
        & Ours & \textbf{92.49 $\pm$ 0.51} & \textbf{795.946 $\pm$ \phantom{1}6.01}& \textbf{20.636 $\pm$ 0.49} \\
    \midrule
    \multirow{4}{*}{\rotatebox{90}{F-MNIST}}
        & VAE & 76.33 $\pm$ 0.89 & 915.934 $\pm$ 10.69 & 50.338 $\pm$ 0.85 \\
        & VAE (noise) & 76.71 $\pm$ 0.58 & 953.350 $\pm$ \phantom{1}7.46 & 52.151 $\pm$ 0.78 \\
        & SE & 77.02 $\pm$ 0.87 & 963.828 $\pm$ 19.64 & 52.896 $\pm$ 1.36 \\
        & Ours & \textbf{77.20 $\pm$ 0.18} & \textbf{821.254 $\pm$ \phantom{1}4.84} & \textbf{45.751 $\pm$ 0.54} \\
    \bottomrule
\end{tabular}
}
\label{tab: ex2}
\end{table}

\subsection{Results and Discussions}
\label{sec: ex-b}
Following the protocols established in Section~\ref{sec: ex-a},
we compare our method with various baselines in practice, including vanilla and noise-augmented VAEs (for definitions of vanilla and noise-augmented, see (A) and (B) in Section~\ref{sec:introduction}) and SE~\cite{Cemgil2020Adversarially} on MNIST and F-MNIST datasets. In our method, the augmented pair $(\firstX, \secondX)$ in Eq.\eqref{eq:haochen data augmentation in original space} is constructed as $(\originalX, \originalX + \bm{\epsilon})$, where $\bm{\epsilon} \sim \mathcal{N}(\bm{0}, 0.05^2 \bm{I})$. We set the variance $\augVar$ in Definition~\ref{dfn:latent generating process} to $0.04^2 \bm{I}$ for F-MNIST and $0.01^2 \bm{I}$ for MNIST. Finally, we adopt cross-entropy metric for the first and second terms in Eq.\eqref{eq:likelihood decomposition via independence property} like the authors of~\cite{DBLP:journals/corr/KingmaW13}. See Appendix~\ref{append:details hyper prameter tuning} for details of hype-parameter tuning.
The results are shown in Figure~\ref{fig: expr} and Table~\ref{tab: ex2}.

As shown in the above results, for vanilla and noise-augmented VAEs, their improvements in robustness under adversarial attacks are limited.
Across datasets and different attacks, our method consistently outperforms SE in adversarial accuracy.
Especially, the performance of SE is highly impacted by strong adversarial attacks in F-MNIST, a relatively complex dataset, while our method is able to sustain its performance under such circumstances, offering robust representations.

In Table~\ref{tab: ex2}, we also report the reconstruction quality in pixel-wise Mean Square Error (MSE) as in~\cite{Cemgil2020Adversarially} and Fr\'echet Inception Distance (FID)~\cite{heusel2017gans}, which captures a deep distance between the ground-truth (clean) images and reconstructed ones. Although SE and noise-augmented VAEs suffer a slight decline in reconstruction fidelity—which can be attributed to their training on altered data and subsequent testing on clean data—our method appears immune to this trend. 

Finally, we emphasize that distance of the paired representation $(\originalZ, \firstZ)$ are well improved compared to VAE and VAE (noise) in Table~\ref{tab: verification of our hypothesis}, suggesting our hypothesis introduced in Section~\ref{sec:introduction} is correct.

\section{Conclusion and Future Work}
\label{sec:conclusion and future work}
Our robust VAE, RAVEN, was proposed based on a novel variational bound via our hypothesis. We empirically confirm that the hypothesis is correct by RAVEN's superior resilience against adversarial attacks. One of our future work is to further investigate our method theoretically and experimentally.

\section*{Acknowledgment}
TK was partially supported by JSPS KAKENHI Grant Number 17H00764, 19H04071, and 20H00576.

\appendix
\section{Lemmas and Propositions}
\label{append: useful lemmas and propositions}

\begin{lem} 
\label{lemma: inverseOfInverseSum}
Let $\bm{X}$ and $\bm{Y}$ be the positive definite symmetric matrix. Then, the following equation holds: $\left(\bm{X}^{-1} + \bm{Y}^{-1}\right)^{-1} = \bm{X} \left(\bm{X} + \bm{Y}\right)^{-1} \bm{Y} = \bm{Y} \left(\bm{X} + \bm{Y}\right)^{-1} \bm{X}$.
\end{lem}
\begin{proof}
    \begin{align*}
        (\text{Middle term})^{-1}
            &= \left(\bm{X} \left(\bm{X} + \bm{Y}\right)^{-1} \bm{Y}\right)^{-1}\\
            &= \bm{Y}^{-1} \left( \bm{X} + \bm{Y} \right) \bm{X}^{-1}\\
            &= \bm{Y}^{-1} \bm{X} \bm{X}^{-1} + \bm{Y}^{-1} \bm{Y} \bm{X}^{-1}\\
            &= \bm{Y}^{-1} + \bm{X}^{-1}\\
            &= (\text{Left term})
    \end{align*}
\end{proof}

\begin{prop}
\label{proposition: productOfNormalPDF}
The product of the two density functions $\Normal{\muA}{\sigmaA}, \Normal{\muB}{\sigmaB}$ can be written as 
$
        \NormalPDF{\vecX}{\muA}{\sigmaA} \cdot \NormalPDF{\vecX}{\muB}{\sigmaB}
            = \NormalPDF{\vecX}{\muC}{\sigmaC} \cdot \NormalPDF{\vecZero}{\muA - \muB}{\sigmaA + \sigmaB},
$
where
$
        \sigmaC = \left( \sigmaA^{-1} + \sigmaB^{-1} \right)^{-1},
        \muC = \sigmaC \left( \sigmaA^{-1} \muA + \sigmaB^{-1} \muB \right)
$.  
\end{prop}
\begin{proof}
Let
$
        P = \NormalPDF{\vecX}{\muA}{\sigmaA} \cdot \NormalPDF{\vecX}{\muB}{\sigmaB}
$, i.e.,
$
P = \frac{1}{\sqrt{\determinant{ 2 \pi \sigmaA} \determinant{2 \pi \sigmaB}}} \exp \left( 
    -\frac{1}{2} A
    \right),
$
where
$
    A:= \quadForm{\sigmaA^{-1}}{(\vecX - \muA)} + \quadForm{\sigmaB^{-1}}{(\vecX - \muB)}.
$
Using Lemma~\ref{lemma: inverseOfInverseSum},
    \begin{align*}
        A
            &= \quadForm{\underbrace{\left( \sigmaA^{-1} + \sigmaB^{-1} \right)}_{=\sigmaC^{-1}} }{\vecX} - 2 \left( \muA^{\top} \sigmaA^{-1} + \muB^{\top} \sigmaB^{-1} \right) \vecX + \left(\quadForm{\sigmaA^{-1}}{\muA} + \quadForm{\sigmaB^{-1}}{\muB} \right)\\
            &= \quadForm{\sigmaC^{-1}}{\vecX} - 2 \underbrace{\left( \muA^{\top} \sigmaA^{-1} + \muB^{\top} \sigmaB^{-1} \right) \sigmaC}_{=\muC^{\top}} \sigmaC^{-1} \vecX + \left(\quadForm{\sigmaA^{-1}}{\muA} + \quadForm{\sigmaB^{-1}}{\muB} \right)\\
            &= \quadForm{\sigmaC^{-1}}{\vecX} - 2 \muC^{\top} \sigmaC^{-1} \vecX + \left(\quadForm{\sigmaA^{-1}}{\muA}+ \quadForm{\sigmaB^{-1}}{\muB} \right)\\
            &= \underbrace{\quadForm{\sigmaC^{-1}}{(\vecX - \muC)}}_{\text{Exponential part of\;} \NormalPDF{\vecX}{\muC}{\sigmaC}} + \underbrace{\bigl(\quadForm{\sigmaA^{-1}}{\muA} + \quadForm{\sigmaB^{-1}}{\muB} - \quadForm{\sigmaC^{-1}}{\muC}\bigr)}_{:= B}.
    \end{align*}
    Thus,
    \begin{align*}
        P
            &= \sqrt{\frac{\determinant{2 \pi \sigmaC}}{\determinant{2 \pi \sigmaA} \determinant{2 \pi \sigmaB}}} \cdot \exp\left(- \frac{1}{2} B\right) \times \underbrace{ \frac{1}{\sqrt{\determinant{2 \pi \sigmaC}}} \exp \left( - \frac{1}{2} \quadForm{\sigmaC^{-1}}{(\vecX - \muC)} \right) }_{=\NormalPDF{\vecX}{\muC}{\sigmaC}} \\
            &= \sqrt{\frac{\determinant{2 \pi \sigmaC}}{\determinant{2 \pi \sigmaA} \determinant{2 \pi \sigmaB}}} \exp\left(-\frac{1}{2} B \right) \cdot \NormalPDF{\vecX}{\muC}{\sigmaC}.
    \end{align*}
    Moreover, 
    \begin{align*}
        B 
            &= \quadForm{\sigmaA^{-1}}{\muA} + \quadForm{\sigmaB^{-1}}{\muB} - \quadForm{\sigmaC^{-1}}{\muC}\\
            &= \quadForm{\sigmaA^{-1}}{\muA} + \quadForm{\sigmaB^{-1}}{\muB} - \quadForm{ \sigmaC \sigmaC^{-1} \sigmaC}{\left( \sigmaA^{-1} \muA + \sigmaB^{-1} \muB \right)}\\
            &= \quadForm{\sigmaA^{-1}}{\muA} + \quadForm{\sigmaB^{-1}}{\muB} - \quadForm{ \sigmaC}{\left( \sigmaA^{-1} \muA + \sigmaB^{-1} \muB \right)}\\
            &= \quadForm{ \underbrace{\left(\sigmaA^{-1} - \sigmaA^{-1} \sigmaC \sigmaA^{-1} \right)}_{=:C} }{\muA} + \quadForm{ \underbrace{\left(\sigmaB^{-1} - \sigmaB^{-1} \sigmaC \sigmaB^{-1} \right)}_{=:D} }{\muB} - 2 \muA^{\top} \underbrace{\sigmaA^{-1} \sigmaC \sigmaB^{-1}}_{=:E} \muB,\\
            &\\
        C
            &= \sigmaA^{-1} - \sigmaA^{-1} \underline{\sigmaC} \sigmaA^{-1}\\
            &= \sigmaA^{-1} - \sigmaA^{-1} \underline{\sigmaA \left( \sigmaA + \sigmaB \right)^{-1} \sigmaB} \sigmaA^{-1}\\
            &= \sigmaA^{-1} - \left( \sigmaA + \sigmaB \right)^{-1} \sigmaB \sigmaA^{-1}\\
            &= \left( \sigmaA + \sigmaB \right)^{-1} \left(\left( \sigmaA + \sigmaB \right) \sigmaA^{-1} - \sigmaB \sigmaA^{-1} \right)\\
            &= \left( \sigmaA + \sigmaB \right)^{-1} \left( \IdentityMatrix + \sigmaB \sigmaA^{-1} - \sigmaB \sigmaA^{-1}\right)\\
            &= \left( \sigmaA + \sigmaB \right)^{-1},\\
            &\\
        D
            &= \sigmaB^{-1} - \sigmaB^{-1} \sigmaC \sigmaB^{-1}\\
            &= \hspace{30pt} \vdots \\
            &= \left( \sigmaA + \sigmaB \right)^{-1},\\
            &\\
        E
            &= \sigmaA^{-1} \underline{\sigmaC} \sigmaB^{-1}\\
            &= \sigmaA^{-1} \underline{\sigmaA \left( \sigmaA + \sigmaB \right)^{-1} \sigmaB} \sigmaB^{-1}\\
            &= \left( \sigmaA + \sigmaB \right)^{-1}.
    \end{align*}
    Therefore, 
    \begin{align*}
        B 
            &= \quadForm{\left( \sigmaA + \sigmaB \right)^{-1}}{\muA} + \quadForm{\left( \sigmaA + \sigmaB \right)^{-1}}{\muB} - 2 \muA^{\top} \left( \sigmaA + \sigmaB \right)^{-1} \muB\\
            &= \underbrace{\quadForm{\left( \sigmaA + \sigmaB \right)^{-1}}{(\muA - \muB)}}_{\text{Exponential part of\;}\NormalPDF{\vecZero}{\muA - \muB}{\sigmaA + \sigmaB}}.
    \end{align*}
    Since
    $
        \determinant{\sigmaC} = \determinant{\sigmaA \left( \sigmaA + \sigmaB \right)^{-1} \sigmaB} = \frac{\determinant{\sigmaA}\determinant{\sigmaB}}{\determinant{\sigmaA + \sigmaB}},
    $
    \begin{align*}
        P
            &= \sqrt{\frac{\determinant{2 \pi \sigmaC}}{\determinant{2 \pi \sigmaA} \determinant{2 \pi \sigmaB}}} \exp\left(-\frac{1}{2} B \right) \cdot \NormalPDF{\vecX}{\muC}{\sigmaC}\\
            &= \underbrace{\frac{\exp \left( -\frac{1}{2} \quadForm{\sigmaA + \sigmaB}{(\muA - \muB)} \right)}{\sqrt{\determinant{2\pi \left(\sigmaA + \sigmaB\right)}}} }_{=\NormalPDF{\vecZero}{\muA - \muB}{\sigmaA + \sigmaB}} \times \NormalPDF{\vecX}{\muC}{\sigmaC}\\
            &= \NormalPDF{\vecZero}{\muA - \muB}{\sigmaA + \sigmaB} \cdot \NormalPDF{\vecX}{\muC}{\sigmaC}\\
    \end{align*}
\end{proof}
\subsection{Proof of Lemma~\ref{proposition: CrossEntropy}}
\label{append:proof lemma2}
\begin{proof}
    From Lemma~\ref{lemma: quadOfMeanCov}, 
    \begin{align*}
         &\Exp{\bx \sim \Normal{\bmu_{a}}{\bSigma_{a}}}{-\log \NormalPDF{\bx}{\bmu_{b}}{\bSigma_{b}} }\\
            &= 
            \mathbb{E}_{\bx \sim \Normal{\bmu_{a}}{\bSigma_{a}}}\Biggl[
            \frac{1}{2}\quadForm{\bSigma_{b}^{-1}}{(\bx - \bmu_{b})} +  \frac{1}{2}\log \determinant{2 \pi \bSigma_{b}} 
            \Biggr]\\
            &= \frac{1}{2} \Bigl[ \, \Tr \left( \bSigma_{b}^{-1} \bSigma_{a} \right) + \quadForm{\bSigma_{b}^{-1}}{( \bmu_{a} - \bmu_{b})} +\log \left((2\pi)^{d} \cdot \determinant{\bSigma_{b}} \right) \, \Bigr]\\
            &= \frac{1}{2} \Bigl[ \, \Tr \left( \bSigma_{a} \bSigma_{b}^{-1} \right) + \quadForm{\bSigma_{b}^{-1}}{( \bmu_{a} - \bmu_{b})} +\log \determinant{\bSigma_{b}} + d \log 2\pi \, \Bigr].
    \end{align*}
    Let $\bmu_{b} = \bmu_{a}$ and $\bSigma_{b} = \bSigma_{a}$, and then the Shannon entropy is
    \begin{align*} 
        H &= \frac{1}{2} \Bigl\{ \Tr(\bSigma_{a} \bSigma_{a}^{-1}) + \quadForm{\bSigma^{-1}}{(\bmu_{a} - \bmu_{a})} + \log \determinant{\bSigma_{a}} + d \log 2\pi \Bigr\}\\
        &= \frac{1}{2} \left\{ \Tr (\IdentityMatrix) + \log \determinant{\bSigma_{a}} + d \log 2\pi \right\}\\
        &= \frac{1}{2} \left\{ d + d\log 2 \pi + \log \determinant{\bSigma_{a}}\right\}.
    \end{align*}
\end{proof}

\subsection{Proof of Proposition~\ref{prop: prior latent joint as closed form}}
\label{append:proof of closed form based on N(0,I) p_tilde_z}
\begin{proof}
    From Proposition~\ref{proposition: productOfNormalPDF},
    \begin{align*}
        &a(\firstZ|\originalZ) \cdot a(\secondZ|\originalZ)\\
            &= \NormalPDF{\firstZ}{\originalZ}{\augVar} \cdot \NormalPDF{\secondZ}{\originalZ}{\augVar}\\
            &= \NormalPDF{\originalZ}{\firstZ}{\augVar} \cdot \NormalPDF{\originalZ}{\secondZ}{\augVar}\\
            &= \NormalPDF{\originalZ}{\frac{\firstZ + \secondZ}{2}}{\frac{1}{2} \augVar}\cdot \NormalPDF{\vecZero}{\firstZ - \secondZ}{2 \augVar }, 
    \end{align*}
    and then, 
    \begin{equation*}
        \begin{split}
            &\NormalPDF{\originalZ}{\vecZero}{\IdentityMatrix} \cdot \NormalPDF{\originalZ}{\frac{\firstZ + \secondZ}{2}}{\frac{1}{2} \augVar} \\
            &= \NormalPDF{\originalZ}{\mupZ}{\sigmapZ} \cdot \NormalPDF{\vecZero}{\frac{\firstZ + \secondZ}{2} }{\IdentityMatrix + \frac{1}{2} \augVar},
        \end{split}
    \end{equation*}
    where $\sigmapZ=\left( \IdentityMatrix + 2 \augVar^{-1} \right)^{-1}\;\text{and}\;\mupZ=\sigmapZ \augVar^{-1}\left(\firstZ + \secondZ\right)$.
    Therefore,
    \begin{align*}
        p_{\paramGMM}(\pairZ)
            &= p_{\paramGMM}(\firstZ, \secondZ)\\
            &= \int a(\firstZ|\originalZ) a(\secondZ|\originalZ) p_{\paramGMM}(\originalZ) \dx{\originalZ}\\
            &= \int \NormalPDF{\originalZ}{\vecZero}{\IdentityMatrix} a(\firstZ | \originalZ) \cdot a(\secondZ | \originalZ) \: \dx{\originalZ}\\
            &=  \int \NormalPDF{\originalZ}{\mupZ}{\sigmapZ} \NormalPDF{\vecZero}{\frac{\firstZ + \secondZ}{2} }{\IdentityMatrix + \frac{1}{2} \: \augVar} \times\NormalPDF{\vecZero}{\firstZ - \secondZ}{2 \augVar} \dx{\originalZ}\\
            &= \: \NormalPDF{\vecZero}{\frac{\firstZ + \secondZ}{2} }{\IdentityMatrix + \frac{1}{2} \augVar} \NormalPDF{\vecZero}{\firstZ - \secondZ}{2 \augVar } \times\underbrace{ \int \NormalPDF{\originalZ}{\mupZ}{\sigmapZ} \dx{\originalZ}}_{=1} \\
            &= \NormalPDF{\vecZero}{\firstZ - \secondZ}{2 \augVar}  \: \NormalPDF{\vecZero}{\frac{\firstZ + \secondZ}{2} }{\IdentityMatrix + \frac{1}{2} \augVar} \\
            &= \NormalPDF{\vecZero}{\firstZ - \secondZ}{2 \augVar}  \, \NormalPDF{\frac{\firstZ + \secondZ}{2}}{\vecZero}{\IdentityMatrix + \frac{1}{2} \augVar}.
    \end{align*}
\end{proof}

\subsection{Proof of Proposition~\ref{prop: mean vector terms}}
\label{append:proof of mean vectors terms}
By Corollary 1.7.1 in~\cite{schott2016matrix},
$
    \left(2 \boldsymbol{I}+\boldsymbol{\Sigma}_{\text {aug }}\right)^{-1}=\boldsymbol{\Sigma}_{\text {aug }}^{-1}-\boldsymbol{\Sigma}_{\text {aug }}^{-1}\left(2^{-1} \boldsymbol{I}+\boldsymbol{\Sigma}_{\text {aug }}^{-1}\right)^{-1} \boldsymbol{\Sigma}_{\text {aug }}^{-1}
$, and note that
\begin{equation*}
 \left(\boldsymbol{\mu}_{\boldsymbol{x}}-\boldsymbol{\mu}_{\boldsymbol{x}^{\prime}}\right)^{\top} \boldsymbol{\Sigma}_{\text {aug }}^{-1}\left(\boldsymbol{\mu}_{\boldsymbol{x}}\!-\!\boldsymbol{\mu}_{\boldsymbol{x}^{\prime}}\right)\!+\!\left(\boldsymbol{\mu}_{\boldsymbol{x}}\!+\!\boldsymbol{\mu}_{\boldsymbol{x}^{\prime}}\right)^{\top} \boldsymbol{\Sigma}_{\text {aug }}^{-1}\left(\boldsymbol{\mu}_{\boldsymbol{x}}+\boldsymbol{\mu}_{\boldsymbol{x}^{\prime}}\right) =2\left(\boldsymbol{\mu}_{\boldsymbol{x}}^{\top} \boldsymbol{\Sigma}_{\text {aug }}^{-1} \boldsymbol{\mu}_{\boldsymbol{x}}+\boldsymbol{\mu}_{\boldsymbol{x}^{\prime}}^{\top} \boldsymbol{\Sigma}_{\text {aug }}^{-1} \boldsymbol{\mu}_{\boldsymbol{x}^{\prime}}\right).
\end{equation*}
The result follows from the above two equations.

\section{Gaussian Mixture Case}
\label{append: GMM case}

\begin{prop}
\label{prop: prior latent joint as closed form based on GMM}
Let us consider the case that $p(\originalZ)$ in Definition~\ref{dfn:latent generating process} is defined by a Gaussian Mixture Model (GMM): $p_{\paramGMM}(\originalZ) = \sum_{c=1}^{C} \gmmCoef{c} \, \NormalPDF{\originalZ}{\gmmMean{c}}{\gmmVar{c}}$. 
For the definition of $\NormalPDF{\originalZ}{\gmmMean{c}}{\gmmVar{c}}$, see Eq.\eqref{eq: multivariate normal distribution}, and $\pi_c$ is the $c$-th wight satisfying $\pi_c \in [0,1]\;\text{and}\;\sum_{c=1}^C \pi_c =1$. The set $\Set{\left( \gmmCoef{c}, \gmmMean{c}, \gmmVar{c}\right)}_{c=1}^{C}$ is trainable, and it is denoted by $\paramGMM$.
Then, the prior $p(\pairZ)$ in Definition~\ref{dfn:latent generating process}, takes on the form of
\begin{equation} 
\label{equation:augmented prior with GMM}
\begin{split}
    \NormalPDF{\vecZero}{\firstZ - \secondZ}{2 \augVar} \sum_{c = 1}^{C} \gmmCoef{c} \, \NormalPDF{\frac{\firstZ + \secondZ}{2}}{\gmmMean{c}}{\gmmVar{c} + \frac{1}{2} \augVar}.
\end{split}
\end{equation}
\end{prop}
\begin{proof}
    From Proposition~\ref{proposition: productOfNormalPDF},
    \begin{align*}
        &a(\firstZ|\originalZ) \cdot a(\secondZ|\originalZ)\\
            &= \NormalPDF{\firstZ}{\originalZ}{\augVar} \cdot \NormalPDF{\secondZ}{\originalZ}{\augVar}\\
            &= \NormalPDF{\originalZ}{\firstZ}{\augVar} \cdot \NormalPDF{\originalZ}{\secondZ}{\augVar}\\
            &= \NormalPDF{\originalZ}{\frac{\firstZ + \secondZ}{2}}{\frac{1}{2} \augVar} \cdot \NormalPDF{\vecZero}{\firstZ - \secondZ}{2 \augVar },
    \end{align*}
    and then,
    \begin{equation*}
        \begin{split}
            &\NormalPDF{\originalZ}{\gmmMean{c}}{\gmmVar{c}} \cdot \NormalPDF{\originalZ}{\frac{\firstZ + \secondZ}{2}}{\frac{1}{2} \augVar} \\
            &= \NormalPDF{\originalZ}{\mupZ}{\sigmapZ} \cdot \NormalPDF{\vecZero}{\frac{\firstZ + \secondZ}{2} - \muC}{\sigmaC + \frac{1}{2} \augVar},
        \end{split}
    \end{equation*}
    where $\sigmapZ = \left( \gmmVar{c}^{-1} + 2 \augVar^{-1} \right)^{-1}$ and $
        \mupZ = \sigmapZ \left( \gmmVar{c}^{-1} \gmmMean{c} + \augVar^{-1}\left(\firstZ + \secondZ\right) \right)$.
    Therefore,
    \begin{align*}
        p_{\paramGMM}(\pairZ)
            &= p_{\paramGMM}(\firstZ, \secondZ)\\
            &= \int a(\firstZ|\originalZ) a(\secondZ|\originalZ) p_{\paramGMM}(\originalZ) \dx{\originalZ}\\
            &= \sum_{c=1}^{C} \gmmCoef{c} \int \NormalPDF{\originalZ}{\gmmMean{c}}{\gmmVar{c}} a(\firstZ | \originalZ) \cdot a(\secondZ | \originalZ) \: \dx{\originalZ}\\
            &= \sum_{c = 1}^{C} \gmmCoef{c} \int \NormalPDF{\originalZ}{\mupZ}{\sigmapZ} \times\NormalPDF{\vecZero}{\frac{\firstZ + \secondZ}{2} - \gmmMean{c}}{\gmmVar{c} + \frac{1}{2} \: \augVar} \\
            &\quad\quad\quad\quad\quad\quad\quad\quad\quad\times\NormalPDF{\vecZero}{\firstZ - \secondZ}{2 \augVar} \dx{\originalZ}\\
            &= \sum_{c = 1}^{C} \gmmCoef{c} \: \NormalPDF{\vecZero}{\frac{\firstZ + \secondZ}{2} - \gmmMean{c}}{\gmmVar{c} + \frac{1}{2} \augVar} \\
            &\quad\quad\quad\quad\quad\quad\quad\quad\quad\times\NormalPDF{\vecZero}{\firstZ - \secondZ}{2 \augVar } \times\underbrace{ \int \NormalPDF{\originalZ}{\mupZ}{\sigmapZ} \dx{\originalZ}}_{=1} \\
            &= \NormalPDF{\vecZero}{\firstZ - \secondZ}{2 \augVar} \times\sum_{i = 1}^{C} \pi_{i} \: \NormalPDF{\vecZero}{\frac{\firstZ + \secondZ}{2} - \gmmMean{c} }{\gmmVar{c} + \frac{1}{2} \augVar} \\
            &= \NormalPDF{\vecZero}{\firstZ - \secondZ}{2 \augVar} \times\sum_{c=1}^{C} \gmmCoef{c} \, \NormalPDF{\frac{\firstZ + \secondZ}{2}}{\gmmMean{c}}{\gmmVar{c} + \frac{1}{2} \augVar}.
    \end{align*}
\end{proof}

\begin{thm}
\label{theorem: augmented ELBO via GMM}
Set Eq.\eqref{equation:augmented prior with GMM} to $p(\Tilde{\bm{z}})$ of Eq.\eqref{eq:elbo paired joint}, then
the following equation holds:
\begin{equation*}
    \begin{aligned}
        &\mathcal{L}\left(\Bar{\firstX}; \bm{\phi}, \bm{\theta}\right)\\
        &= \Exp{q_{\paramEnc}(\pairZ | \pairX)}{\log p_{\paramDec}(\pairX | \pairZ)} + \Exp{q_{\paramEnc}(\pairZ | \pairX)}{\log \sum_{c = 1}^{C} \gmmCoef{c} \: \NormalPDF{\frac{\firstZ + \secondZ}{2}}{\gmmMean{c}}{\gmmVar{c} + \frac{1}{2}\augVar} } \\
         &\;\;\;\;\; - \frac{1}{4} \Bigl\{ \Tr{\left(\augVar^{-1}(\varEnc{\firstX} + \varEnc{\secondX})\right)} + \quadForm{\augVar^{-1}}{(\meanEnc{\firstX} - \meanEnc{\secondX})} \Bigr\}\\
         &\;\;\;\;\;\; + \frac{1}{2} \left(\log \determinant{\varEnc{\firstX} } + \log \determinant{\varEnc{\secondX}} \right) + d + \frac{d}{2} \log \pi- \frac{1}{2} \log \determinant{\augVar}.
    \end{aligned}
\end{equation*}
\end{thm}
\begin{proof}
    From Lemma~\ref{lemma: quadOfMeanCov} and~\ref{proposition: CrossEntropy},
    \begin{align*}
        &\Exp{\firstZ \sim \Normal{\meanEnc{\firstX}}{\varEnc{\firstX}}}{\log \NormalPDF{\vecZero}{\firstZ - \secondZ}{2\augVar}}\\
        &= -\frac{1}{4} \left(\Tr{\left( \augVar^{-1} \varEnc{\firstX}\right)} + \quadForm{\augVar^{-1}}{(\meanEnc{\firstX} - \secondZ)} \right) - \frac{1}{2} \log \determinant{4\pi \augVar}.
    \end{align*}
    Using the above equation,
    \begin{align*}
        &\Exp{\pairZ \sim q_{\paramEnc}(\pairZ | \pairX)}{\log \NormalPDF{\vecZero}{\firstZ - \secondZ}{2\augVar}}\\
            &= -\frac{1}{4} \Bigl\{ \Tr{(\augVar^{-1}\varEnc{\firstX})} + \Exp{\secondZ \sim \Normal{\meanEnc{\secondX}}{\varEnc{\secondX}}}{ \quadForm{\augVar^{-1}}{(\meanEnc{\firstX} - \secondZ)}} \Bigr\} -\frac{1}{2} \log \determinant{4 \pi \augVar} \\
            &= - \frac{1}{4} \Bigl\{\Tr{(\augVar^{-1}\varEnc{\firstX})} + \Tr{(\augVar^{-1}\varEnc{\secondX})} + \quadForm{\augVar^{-1}}{(\meanEnc{\firstX} - \meanEnc{\secondX})} \Bigr\} \\
            &\quad\quad\quad\quad\quad\quad\quad\quad\quad- \frac{1}{2}\left(\log \determinant{\augVar} + d \log 4 \pi \right)\\
            &= -\frac{1}{4}\Tr{\left(\augVar^{-1}(\varEnc{\firstX}\!+\!\varEnc{\secondX})\!+\! \quadForm{\augVar}{(\meanEnc{\firstX}\!-\!\meanEnc{\secondX})} \right)}  - \frac{1}{2}\left(\log \determinant{\augVar} + d \log 4 \pi \right),
    \end{align*}
    and
    \begin{align*}
        &\Exp{q_{\paramEnc}(\pairZ|\pairX)}{\log p_{\paramGMM}(\pairZ) } \\
        &=\mathbb{E}_{q_{\paramEnc}(\pairZ|\pairX)}\Biggl[
        \log \sum_{c=1}^{C} \gmmCoef{c} \, \NormalPDF{\frac{\firstZ + \secondZ}{2}}{\gmmMean{c}}{\gmmVar{c} + \frac{1}{2} \augVar} +
        \log \NormalPDF{\vecZero}{\firstZ - \secondZ}{2 \augVar}
        \Biggr]\\
            &= \Exp{q_{\paramEnc}(\pairZ|\pairX)}{\log \sum_{c=1}^{C} \gmmCoef{c} \, \NormalPDF{\frac{\firstZ + \secondZ}{2}}{\gmmMean{c}}{ \gmmVar{c} + \frac{1}{2} \augVar }}  - \frac{1}{2}(\log \determinant{\augVar} + d \log 4\pi)\\
            &\hspace{20pt} - \frac{1}{4} \Bigl\{ \Tr{\left(\augVar^{-1}(\varEnc{\firstX} + \varEnc{\secondX}) \right)} + 
            \quadForm{\augVar^{-1}}{(\meanEnc{\firstX} - \meanEnc{\secondX})} \Bigr\}.
    \end{align*}
    Therefore,
    \begin{equation*}
        \begin{aligned}
            \mathcal{L}\left(\Bar{\firstX}; \bm{\phi}, \bm{\theta}\right)
                &= \Exp{q_{\paramEnc}(\pairZ | \pairX)}{\log p_{\paramDec}(\pairX | \pairZ)} - \KLDiv{q_{\paramEnc}(\pairZ | \pairX) }{p_{\paramGMM}(\pairZ)}\\
                &= \Exp{q_{\paramEnc}(\pairZ | \pairX)}{\log p_{\paramDec}(\pairX | \pairZ)} + \Exp{q_{\paramEnc}(\pairZ | \pairX)}{\log p_{\paramGMM}(\pairZ)} + H\left(q_{\paramEnc}(\pairZ | \pairX) \right) \\
                 &= \Exp{q_{\paramEnc}(\pairZ | \pairX)}{\log p_{\paramDec}(\pairX | \pairZ)} + \Exp{q_{\paramEnc}(\pairZ | \pairX)}{\log \sum_{c = 1}^{C} \gmmCoef{c} \: \NormalPDF{\frac{\firstZ + \secondZ}{2}}{\gmmMean{c}}{\gmmVar{c} + \frac{1}{2}\augVar} } \\
                 &\hspace{20pt} - \frac{1}{4} \Bigl\{ \Tr{\left(\augVar^{-1}(\varEnc{\firstX} + \varEnc{\secondX})\right)} + \quadForm{\augVar^{-1}}{(\meanEnc{\firstX} - \meanEnc{\secondX})} \Bigr\}\\
                 &\hspace{30pt} + \frac{1}{2} \left(\log \determinant{\varEnc{\firstX} } + \log \determinant{\varEnc{\secondX}} \right) + d\left(1 + \frac{1}{2} \log \pi\right) - \frac{1}{2} \log \determinant{\augVar}.
        \end{aligned}
    \end{equation*}
\end{proof}

\section{Further Details of Numerical Experiments}
\label{append: further details of numerical experiments}
Let $\originalX^\text{tr}_i$ and $\originalX^\text{ts}_j$ denote $i$-th original data in training dataset and $j$-th original data in test dataset, respectively. In addition, let $\mathcal{D}^{\rm tr} = \{\originalX^\text{tr}_i\}_{i=1}^n$ and $\mathcal{D}^{\rm ts} = \{\originalX^\text{ts}_j\}_{j=1}^m$, where $n$ (resp. $m$) denotes the number of training (resp. test) data points.

\subsection{Preliminary Experiments}
\label{apped: details preliminary experiments}
Using the original MNIST dataset~\cite{lecun1998gradient},
we firstly evaluated the performance of the following two trained model:
\begin{description}
    \item[(A)] a VAE trained by $\mathcal{D}^{\rm tr}$,
    \item[(B)] a VAE trained by $\{\originalX^\text{tr}_i\}_{i=1}^n \cup \{\firstX^\text{tr}_i\}_{i=1}^n$, where $\firstX^\text{tr}_i=\originalX^\text{tr}_i + \bm{\epsilon}_i$, and $\bm{\epsilon}_i$ is $i$-th random Gaussian noise.
\end{description}
The dimension of the latent space was ten for both VAEs, and both architectures (encoder and decoder) were the exactly same as described at the beginning of Section~\ref{subsubsec: Experiment protocols}.   
Additionally, $(n,m)=(60000,10000)$ with the original MNIST. 
Let 
$\textrm{Enc}^{\rm (A)}_{\bm{\phi}^\ast}$ (resp. $\textrm{Enc}^{\rm (B)}_{\bm{\phi}^\ast}$) denote the trained encoder of (A) (resp. (B)).

As shown in a) and b) of~Figure~\ref{fig: expr}, the accuracies of (A) and (B) were almost the same. 
We then investigated the following two sets of the latent variables: $\mathcal{Z}_{\rm (A)} = \left\{\originalZ^{\rm ts}_{j, {\rm (A)}}\right\}_{j=1}^m \cup \left\{\firstZ^{\rm ts}_{j, {\rm (A)}}\right\}_{j=1}^m$ and $\mathcal{Z}_{\rm (B)} = \left\{\originalZ^{\rm ts}_{j, {\rm (B)}}\right\}_{j=1}^m \cup \left\{\firstZ^{\rm ts}_{j, {\rm (B)}}\right\}_{j=1}^m$, where 
$\originalZ^{\rm ts}_{j, {\rm (\cdot)}} := \textrm{Enc}^{\rm (\cdot)}_{\bm{\phi}^\ast}\left(\originalX^\text{ts}_j\right)$ and $\firstZ^{\rm ts}_{j,{\rm (\cdot)}} := \textrm{Enc}^{\rm (\cdot)}_{\bm{\phi}^\ast}\left(\originalX^\text{ts}_j + \bm{\epsilon}_j \right)$. In Figure~\ref{fig: visualization of latent via vae with noisy data}, the t-SNE visualization of $\mathcal{Z}_{\rm (A)}$ and $\mathcal{Z}_{\rm (B)}$ are shown in (A) and (B), respectively.
Consider the set 
$S := \left\{ \left\|\originalZ^{\rm ts}_{j} - \firstZ^{\rm ts}_{j}\right\|_2^2:j=1,...,m\right\}$ on $\mathcal{Z}$. 
Then, the average and std over $S$ for (A) and (B) are shown in the first and second column in Table~\ref{tab: verification of our hypothesis}.

\subsection{Hyper-Parameter Tuning and Model Architecture}
\label{append:details hyper prameter tuning}
In Section~\ref{sec:numerical experiments}, we fixed the dimension of the latent space to $10$ for all methods. In addition, the hyper-parameters in all the existing methods were sufficiently tuned in our preliminary experiments using both MNIST and F-MNIST. 
For our method, Raven, the primary hyperparameter $\augVar$  is set to be $\sigma^2 I $. We explored a range for $\sigma \in \{0.1, 0.2, 0.4, 0.1, 0.2, 0.4, 1\}$ and found the best performing $\sigma=0.04$ for F-MNIST and $\sigma=0.01$ for MNIST, based on the metrics described at the end of Section~\ref{subsubsec: Experiment protocols}.

In our experiments, we used a fully connected feed-forward neural network with layers arranged as $784-500-250-10$ (encoder), $10-250-500-784$ (decoder) with PReLU activation functions~\cite{he2015delving}.

For SE, the key hyperparameter is $\eta$, which controls the coupling strength for the regularized Wasserstein distance between the latent representations of the original data points and adversarial data points. From 
$\eta \in \{ 50,100,250,500,1000,2500,5000,10000\}$, we found $\eta = 500$ performs the best for MNIST and $\eta=5000$ for F-MNIST.

\bibliographystyle{IEEEtran}
\bibliography{allref_camera-ready}

\end{document}